\documentclass{article}

\usepackage{microtype}
\usepackage{graphicx}
\usepackage{amsfonts}
\usepackage{amsmath}
\usepackage{enumitem}
\usepackage{pgfkeys}
\usepackage{tcolorbox}
\tcbuselibrary{skins}
\definecolor{promptbg}{RGB}{245,245,245} 
\usepackage{subcaption}
\usepackage{booktabs} 
\usepackage{tabularx}

\usepackage{hyperref}

\usepackage[preprint]{icml2026}

\usepackage{amsmath}
\usepackage{amssymb}
\usepackage{mathtools}
\usepackage{amsthm}

\usepackage[capitalize,noabbrev]{cleveref}

\theoremstyle{plain}
\newtheorem{theorem}{Theorem}[section]
\newtheorem{proposition}[theorem]{Proposition}
\newtheorem{lemma}[theorem]{Lemma}

\theoremstyle{definition}
\newtheorem{definition}{Definition}
\newtheorem{assumption}{Assumption}
\theoremstyle{remark}

\usepackage[textsize=tiny]{todonotes}

\icmltitlerunning{LLM Personas as a Substitute for Field Experiments in Method Benchmarking}

\begin{document}

\twocolumn[
  \icmltitle{LLM Personas as a Substitute for Field Experiments in Method Benchmarking}



  \icmlsetsymbol{equal}{*}

  \begin{icmlauthorlist}
    \icmlauthor{Enoch Hyunwook Kang}{yyy}
  \end{icmlauthorlist}

  \icmlaffiliation{yyy}{Foster School of Business, University of Washington, Seattle, USA}

  \icmlcorrespondingauthor{Enoch  Hyunwook Kang}{ehwkang@uw.edu}

  \icmlkeywords{Machine Learning, ICML}

  \vskip 0.3in
]



\printAffiliationsAndNotice{}  

\begin{abstract}
Field experiments (A/B tests) are often the most credible benchmark for methods (algorithms) in societal systems, but their cost and latency bottleneck rapid methodological progress. LLM-based persona simulation offers a cheap synthetic alternative, yet it is unclear whether replacing humans with personas preserves the benchmark interface that adaptive methods optimize against. We prove an if-and-only-if characterization: when (i) methods observe only the aggregate outcome (aggregate-only observation) and (ii) evaluation depends only on the submitted artifact and not on the method's identity or provenance (method-blind evaluation), swapping humans for personas is just panel change from the method's point of view, indistinguishable from changing the evaluation population (e.g., New York to Jakarta). Furthermore, we move from validity to usefulness: we define an information-theoretic discriminability of the induced aggregate channel and show that making persona benchmarking as decision-relevant as a field experiment is fundamentally a sample-size question, yielding explicit bounds on the number of independent persona evaluations required to reliably distinguish meaningfully different methods at a chosen resolution.
\end{abstract}

\section{Introduction}
One of the recurring lessons from machine learning is that improvements in methods accelerate dramatically when we can iterate quickly against cheap and reliable evaluation benchmarks \cite{blum2015ladder, zaharia2018accelerating,liao2021we, d2022underspecification, bommasani2023holistic, miller2024adding, abbas2025developing}. The literature advances through an iterative process in which researchers propose a variant, test it against a benchmark, inspect failures, and refine. As a result, the limiting factor in progress is often the latency and reliability of the evaluation feedback loop rather than the availability of new ideas \cite{xin2018developers}.

In many societal applications, such as pricing and matching policies, ad allocation, marketplace design, and behavioral interventions, the closest analogue to a benchmark is a field experiment \cite{harrison2004field, duflo2017economist}, also known as A/B testing \cite{kohavi2020trustworthy}. Here, competing methods can be compared by deploying them through estimating their causal effects on agreed-upon outcome metrics. Repeating such experiments across settings allows researchers to assess not only average performance but also robustness and behavior under adaptation. 

However, field experiments are costly to set up and slow to run \cite{diamond1986laboratory, fisher1992arrangement, paluck2014field, samek2019advantages}. Researchers typically must first collaborate with an organization, propose an intervention, negotiate approval across multiple stakeholders, and spend months implementing instrumentation and experimental infrastructure before any data are collected, and then wait to collect data until enough statistical power is gathered. Moreover, feasible experiments are often constrained by the organization’s interests, risk tolerance, and resource limitations, which restrict what methods can be tested and how frequently experiments can be repeated \cite{bandiera2011field}.

Recently, LLM-based persona simulation has shown potential to serve as a cheaper synthetic field laboratory \cite{toubia2025database, peng2025mega}. In LLM-based persona simulation, a large language model is conditioned on explicit persona descriptions, such as survey responses that capture demographic attributes, preferences, roles, or beliefs. This is then used to generate individual-level evaluations or responses to proposed policies, content, or system outputs. These simulated micro-level responses can then be aggregated into a single evaluation signal. Because such simulations are inexpensive and fast to run at scale, they make it feasible to test and iterate on methods without field experimentation or A/B testing.

However, whether LLM-based persona simulation can indeed reliably substitute for field experimentation or A/B testing for methodology testing remains largely unknown. This uncertainty is exacerbated by recent studies that provide negative results in applications beyond methodology testing, particularly when we ask causal questions to support external validity claims \cite{gui2023challenge, li2025llm, peng2025mega, gui2025leveraging}. Manipulating prompts intended to represent a treatment can inadvertently shift other latent aspects of the scenario, inducing confounding. 

Therefore, we face the following fundamental question:

\begin{center}
    \textit{When can LLM-based persona simulation serve as a drop-in substitute for a field experiment (or A/B test) as a benchmark for comparing methods?}
\end{center}
In other words, under what protocol conditions is replacing humans with personas, from the benchmark's perspective, equivalent to changing only the evaluation population (as when we change from the New York population to the Jakarta population)? 

At first glance, this may sound like a question that can be empirically/experimentally validated: one might test persona benchmarking by checking whether persona scores correlate with human A/B test outcomes across a set of methods. However, such a correlation cannot validate or invalidate the "drop-in substitute'' claim. The claim is not about agreement of scores between personas and humans; it is about whether swapping the evaluator preserves the \emph{benchmark interface} that methods optimize against. If personas simply represent a different evaluation population, then even in the best-case scenario, the induced score distribution can (and generally should) differ from the human one, just as the same method can score differently in New York than in Jakarta, so matching outcomes is neither necessary nor expected. For these reasons, the question we study is best treated as an identification question.

In this paper, we give the exact theoretical characterization of when this benchmark-level equivalence holds, by identifying two benchmark-hygiene conditions: 
\vspace{-0.2cm}
\begin{itemize}[leftmargin=0.55cm,itemsep=1pt]
    \item [(i)] \emph{aggregate-only observation (AO):} Each method observes only the final aggregate
score (and not the individual-level responses or individual identities),
\item [(ii)] \emph{method-blind  evaluation (MB):} For each method, the distribution of the returned score depends only on what was submitted and not on which
training/optimization procedure produced. 
\end{itemize}

We show that these two conditions are jointly \textit{necessary and sufficient} for swapping human
evaluation for persona evaluation to be equivalent to
changing only the evaluation population, e.g., changing from the New York population to the Jakarta population. In other words, swapping humans for personas is indistinguishable by the method
interface from an ordinary change of the evaluation population.

In addition to this identification result, we extend the identification discussion into a notion of \emph{usefulness}: once (AO)+(MB) makes persona benchmarking identification-valid (i.e., ``just panel change'' on the method’s interface), the remaining question is whether the induced aggregate channel \(Q_{\mathrm{pers}}(\cdot\mid w)\) is informative enough to distinguish and optimize meaningfully different methods. We formalize this via an information-based discriminability parameter (worst-case KL separation at a pre-specified resolution) and show it yields an explicit sample-complexity rule for how many independent persona evaluations (equivalently, what effective persona dataset/panel size) are required for reliable method comparison. In this sense, beyond enforcing (AO)+(MB), ``persona quality'' becomes a measurable budget question: is the persona panel large enough to resolve the improvements we care about?
 
The rest of the paper is organized as follows. Section~\ref{sec:related} reviews related work. Section~\ref{sec:setup} formalizes the benchmarking setup and states the (AO) and (MB) conditions. Section~\ref{sec:identification} gives the main identification discussion, proving that if and only if (AO) and (MB) holds, persona vs.\ human evaluation is just panel change (JPC). Section~\ref{sec:usefulness} extends the identification discussion to a usefulness discussion that relates to sample complexity. The Appendix contains deferred proofs.

\section{Related works}
\label{sec:related}

We discuss papers in the literature on benchmarks in AI in
computer science and field experiment design that are related to the key contributions of our paper.

Aggregate-only observation (AO), a key condition for using persona simulation rather than field experiments for method benchmarking, has also appeared in the literature on benchmark gaming \cite{blum2015ladder, hardt2017climbing, feldman2019advantages, biderman2024trenches} as a baseline for preventing leaderboard gaming via adaptive overfitting. In econometrics, treatments of randomized experiments formalize the evaluation protocol as an assignment-and-aggregation procedure, clarifying what is identified from the realized (often aggregated) outcomes \cite{atheyimbens2017econometrics}. 

Another key condition, method-blind evaluation (MB), has been experimentally motivated: people rate the same advice differently when it is labeled as AI rather than human/crowd-sourced \cite{bogert2022human, osborne2025me}; many large-scale evaluation platforms implement partial blinding (e.g., anonymous pairwise comparisons) precisely to reduce provenance effects \cite{chiang2024chatbotarena}. \citet{dominguezolmedo2024questioning} shows that survey-style elicitation of LLMs exhibits strong ordering and labeling effects that can dominate conclusions. The econometrics literature has also shown that outcomes can shift when participants condition on contextual/provenance cues, motivating blinding and stable scoring protocols \cite{dequidt2018demand,levitt2011hawthorne}.

Related to quantifying a benchmark's usefulness, \citet{madaan2024variance} studied variance-focused analyses that quantify the extent to which benchmark scores fluctuate during evaluation. \citet{heineman2025signal} discusses ``signal vs.\ noise'' perspectives and argue that benchmarks with higher signal-to-noise ratios are more reliable for model-selection decisions.  Compared to these papers, we focus on proposing a benchmark-internal information measure (minimum KL separation) that directly yields sample-size scaling for reliable method comparison on the induced aggregate channel, clarifying how large a persona dataset is practically enough.

\section{Setup: Benchmarking experimentation}
\label{sec:setup}

\subsection{Configurations and artifacts.}

Let $\Theta$ denote the space of tunable configuration variables; this includes all controllable degrees of freedom that specify
a system or a procedure (e.g., model weights, prompts/context, hyperparameters, decoding rules, tool policies,
memory policies, data curation choices, or post-processing rules). A single candidate \emph{configuration} is a choice $\theta\in\Theta$. A submitted \textit{artifact} is what you hand to the benchmark, i.e., the externally visible object that the benchmark evaluates. We model this via an
\emph{artifact map} $g:\Theta\to\mathcal W$ and write
\[
w = g(\theta)\in\mathcal W.
\]
Possible choice of the artifact space $\mathcal{W}$ encompass:
(i) a single output for a fixed input,
(ii) an output distribution (stochastic method),
(iii) a full interaction method mapping contexts to actions/outputs,
(iv) a rollout distribution of an agent interacting with tools or environments.

\paragraph{Method as a configuration optimizer.} \;We model method benchmarking as a ``submit--observe'' loop. A \emph{method} (algorithm) may be either
\emph{non-adaptive} (a single-shot submission, or a fixed distribution over submissions) or \emph{adaptive}
(updating submissions based on past benchmark feedback). Concretely, at each round $t=1,2,\dots,T$, the method chooses a configuration $\theta_t\in\Theta$
(equivalently an artifact $w_t := g(\theta_t)\in\mathcal{W}$), submits $w_t$ to an \emph{evaluator},
and receives a feedback observation $o_t$ taking values in some observation space $\mathcal{O}$. Formally, a method (or method) $\mathcal{A}$ is modeled as a (possibly randomized)
procedure that induces a decision kernel $\pi_t$ over configurations at each round $t$,  
such that
\[
\theta_t \sim \pi_t(\cdot \mid H_{t-1}, S),
\]
where $H_{t-1}$ is the method's observable history before round $t$, i.e., 
\[
H_{t-1} := \{(\theta_\tau,o_\tau)\}_{\tau=1}^{t-1}
`\]
and $S$ denotes any \emph{side information} available before benchmarking begins
(e.g., offline datasets, pretrained weights, logs, simulators).

Note that such a definition of an method contains:
\begin{itemize}[leftmargin=0.55cm,itemsep=1pt]
\item \textbf{Offline alignment} (e.g., DPO \cite{rafailov2023direct}, SFT \cite{wei2022finetuned} ) as the special case where $\pi_t$ does not depend on $H_{t-1}$
(or where $T=1$).
\item \textbf{Online alignment} (e.g., RLHF \cite{ouyang2022instructgpt}) as the case where $\pi_t$ adapts to past
feedback $o_{1:t-1}$.
\item \textbf{Hyperparameter tuning / AutoML / architecture search} \cite{he2021automl, ren2021comprehensive} as black-box optimization over $\Theta$ using
benchmark scores, where $\theta_t$ encodes a full training-and-deployment recipe, and $\pi_t$ implements a sequential search
procedure.

\item \textbf{Prompt and system configuration search} \cite{pryzant2023automatic, kang2025bayesian} as the case where $\Theta$
indexes prompts, system messages, tool-use policies, retrieval and memory settings, and post-processing rules.

\item \textbf{Data-centric training recipe search} \cite{zha2025data} as the case where $\Theta$ includes dataset construction and
curation choices (filtering, reweighting, mixing, synthetic-data generation policies).
\end{itemize}

\subsection{Evaluation as panel $\times$ instrument $\times$ aggregation}
\label{ssec:panel-instrument}

We model an evaluation as a two-stage process: many \emph{micro-level judgments} are first produced,
and these are then compressed into a single \emph{aggregate feedback signal} that the method
actually observes. Under this model, an \emph{evaluation setup} (or simply an \emph{evaluator}) is fully specified as
\[
(P,I,\Gamma,L),
\]
which we will treat as a primitive object throughout.

\paragraph{Panel ($P$).}
A \emph{panel} is a population of evaluators, either humans or LLM personas.
Formally, let $\mathcal{P}$ denote the panel space and let $P$ be a distribution over $\mathcal{P}$.
Each evaluation call draws a fresh panel of $L$ independent evaluators
\[
p_1,\dots,p_L \overset{\text{i.i.d.}}{\sim} P.
\]

\paragraph{Micro-instrument ($I$). }
Given an artifact $w\in\mathcal{W}$ and a panel member $p\in\mathcal{P}$, the \emph{micro-instrument}
produces an individual response.
Formally, this is a conditional distribution
\begin{equation}
    I(\cdot \mid w,p)
    \quad\text{over a micro-response space }\mathcal{Z}.
\end{equation}
where the micro-response space $\mathcal{Z}$ can be a Likert score, a binary preference, or a short textual judgment.
Each evaluator responds independently:
\[
Z_\ell \sim I(\cdot\mid w,p_\ell), \qquad \ell=1,\dots,L.
\]

\paragraph{Aggregation into observed feedback ($\Gamma$, $L$).}
Finally, the benchmark aggregates the $L$ micro-responses into a single observable output.
This is captured by a deterministic aggregation map
\begin{equation}
    \Gamma:\mathcal{Z}^L\to\mathcal{O},
\end{equation}
where $\mathcal{O}$ is the feedback observation space.
Typical examples include the mean score, a majority vote, or a pass/fail indicator.
Putting these pieces together, a single evaluation call on artifact $w$ returns the aggregate
\begin{equation}
    o \;=\; \Gamma(Z_1,\dots,Z_L)\in\mathcal{O}.
\end{equation}

\paragraph{What's observed by the method.}
Although evaluation involves panel members and micro-responses, the method never observes them. What it sees is only the induced distribution of the aggregate feedback $o\in\mathcal{O}$.
The tuple $(P,I,\Gamma,L)$ therefore defines a Markov kernel on $\mathcal{O}$:
\begin{align}
    Q_{P,I}(A\mid w)
    \;:=\;
    \mathbb{P}\big(\Gamma(Z_1,\dots,Z_L)\in A \,\big|\, w\big), \label{eq:reducedform} 
    \\
    A\subseteq\mathcal{O}\ \text{measurable}. \notag
\end{align}
Intuitively, $Q_{P,I}(\cdot\mid w)$ is the distribution of the single observable feedback produced
by the whole pipeline ``sample panel $\to$ elicit micro-responses $\to$ aggregate'' when the submitted artifact is $w$.
All identification arguments in the sequel are necessarily about this reduced-form object.

\subsection{Persona benchmark vs.\ human benchmark}
\label{ssec:persona-vs-human}

Up to this point, we have described an \emph{evaluation setup} abstractly as a tuple $(P,I,\Gamma,L)$.
Now we instantiate this abstraction in the two cases we want to compare: evaluation by humans versus
evaluation by LLM personas.

\paragraph{Human benchmark.}
In a \emph{human benchmark}, the panel distribution $P_{\mathrm{hum}}$ samples human evaluators, and
the micro-instrument $I_{\mathrm{hum}}$ is the procedure that elicits a micro-response from a human
(e.g., a rating, a preference, or a short written judgment). Together with the same aggregation map
$\Gamma$ and panel size $L$, this induces an observable feedback kernel
\[
Q_{\mathrm{hum}}(\cdot\mid w).
\]

\paragraph{Persona benchmark.}
In a \emph{persona benchmark}, the panel distribution $P_{\mathrm{pers}}$ samples persona profiles
(e.g., demographic or attitudinal descriptors), and the micro-instrument $I_{\mathrm{pers}}$ is
implemented by an LLM judge conditioned on the sampled persona. Using the \emph{same} aggregation
$\Gamma$ and panel size $L$ yields a second observable feedback kernel
\[
Q_{\mathrm{pers}}(\cdot\mid w).
\]

\paragraph{What matters for the method.}
Although these two pipelines differ internally (humans vs.\ personas; human judgments vs.\ LLM judgments),
the method only interacts with each benchmark through the induced distribution of the
\emph{aggregate} feedback. In other words, for the method, the relevant objects are precisely the
two reduced-form kernels $Q_{\mathrm{hum}}(\cdot\mid w)$ and $Q_{\mathrm{pers}}(\cdot\mid w)$.

\subsection{Key assumptions}

In asking the key question in this paper, ``\textit{when can we treat persona evaluation as a clean benchmark interface for comparing methods?}'', we need to discuss two ``benchmark hygiene'' conditions that clarify what information the method
does (and does not) get access to, and whether the benchmark behaves like a well-defined environment
independent of who is playing. They are not substantive modeling
assumptions about humans or LLMs; they are assumptions about what information the benchmark reveals
and how the evaluator behaves as an environment.

\begin{assumption}[Aggregate-only observation (AO)]
\label{ass:AO}
At each round $t$, the method observes only the aggregate feedback $o_t\in\mathcal{O}$.
It does \emph{not} observe the micro-level tuple $(p_1,\dots,p_L, Z_1,\dots,Z_L)$, any panel
identifiers, or any additional side-channel information beyond $o_t$.
\end{assumption}
Intuitively, (AO) says the method sees exactly what a standard leaderboard would show: one score (or label) per submission.
This prevents method's gaming behavior that relies on recognizing individual panelists/personas or exploiting
micro-level structure that would be invisible in the intended benchmark interface.

Before discussing the next assumption, we define a probability measure $\mathbb{P}^{\mathcal{A}}$: for a fixed benchmark implementation, running an method $\mathcal{A}$ induces a probability
measure $\mathbb{P}^{\mathcal{A}}$ over the interaction transcript $(w_1,o_1,w_2,o_2,\dots)$. Under (AO), we can define the method's interaction transcript after $t-1$ rounds as
\[
\widetilde{H}_{t-1}:=\{(w_\tau,o_\tau)\}_{\tau=1}^{t-1},
\]
and its information before choosing $w_t$ is the $\sigma$-field
\[
\mathcal I_{t-1}:=\sigma(S,\widetilde{H}_{t-1}).
\]

\begin{assumption}[Method-blind evaluation (MB)]
\label{ass:MB}
There exists a Markov kernel $Q(\cdot\mid w)$ on $\mathcal O$ such that for every method $\mathcal A$,
every round $t$, and every measurable $A\subseteq\mathcal O$,
\[
\mathbb{P}^{\mathcal{A}}\!\big(o_t\in A \,\big|\, \mathcal I_{t-1},\, w_t\big)
=
Q(A\mid w_t)
\quad \text{a.s.}
\]
\end{assumption}

Intuitively, (MB) is the minimal condition for calling this evaluation setup a \emph{benchmark environment} at all: the evaluator should not care about the identity of the training procedure 
or other metadata, and care only about what was submitted (the artifact). In other words, the benchmark interaction is fully summarized by the reduced-form kernel $Q(\cdot \mid w)$, which is fixed across methods.

\section{Identification: When Is Persona Benchmarking ``Just Panel Change''?}
\label{sec:identification}

From the method's point of view, each benchmark is a black box: it takes an artifact
$w\in\mathcal{W}$ and returns a random aggregate feedback value $o\in\mathcal{O}$.
All the internal structure (panel draws, micro-judgments, and aggregation) has already been
compressed into the reduced-form kernels
\[
Q_{\mathrm{pers}}(\cdot\mid w)
\quad\text{and}\quad
Q_{\mathrm{hum}}(\cdot\mid w),
\]
defined in \eqref{eq:reducedform}. In this section, we utilize this intuition to answer to the following question:
\begin{quote}
\emph{When is swapping human evaluation for persona evaluation, as seen through the method's interface,
nothing more than changing the panel $P$?}
\end{quote}

We first formalize what it means to ``only change the evaluation population'' in the
panel--instrument--aggregation model (a \emph{literal panel change}), and then define the corresponding
\emph{interface-level} notion that is relevant for adaptive benchmarking methods (just panel change, JPC).
We then show that for human vs.\ persona benchmarking, JPC holds \emph{if and only if} two benchmark-hygiene
conditions hold: aggregate-only observation (AO) and method-blind evaluation (MB).

This yields an identification result: under (AO)+(MB), swapping humans for personas is indistinguishable on the
method interface from an ordinary change of evaluation population (even though the internal micro-instrument may differ);
conversely, if either condition fails, the swap can change the interface in ways that go beyond ``panel change.''

\subsection{Literal panel change iff just panel change (JPC)}

\begin{definition}[Literal panel change]
\label{def:literal_panel_change}
Fix a panel space $\mathcal P$, micro-response space $\mathcal Z$, aggregation map $\Gamma:\mathcal Z^L\to\mathcal O$ and
panel size $L\in\mathbb N$. Let $I(\cdot\mid w,p)$ be a micro-instrument on $\mathcal Z$.
For two panel distributions $P,P'$ on $\mathcal P$, define two benchmarks
\[
B:=(P,I,\Gamma,L),
\qquad
B':=(P',I,\Gamma,L).
\]
We say $B'$ is obtained from $B$ by a \emph{literal panel change} if the only difference between them is that $P$ is replaced
by $P'$ (i.e., $I,\Gamma,L$ are identical).
\end{definition}

Definition~\ref{def:literal_panel_change} captures the classical ``same survey, different respondents'' intuition: two benchmarks has the same instrument (the question/rubric and how responses are generated), the same aggregation rule, and the same sample size, but we sample from a different population (e.g., one from New York and one from Jakarta). That is, the protocol is unchanged except for the distribution.
over who evaluates.

Swapping humans for personas is not literally a panel change in this narrow sense because the micro-instrument is
implemented differently (humans vs.\ an LLM judge conditioned on a persona).
This motivates an interface-level notion that asks whether the swap \emph{behaves like} a panel change to any method that can
adapt to benchmark feedback; the following \textit{Just panel change (JPC)} definition serves the role. 

\begin{definition}[Just panel change (JPC)]
\label{def:jpc_interface}
Consider two benchmarks $B$ and $B'$ on the same artifact space $\mathcal W$ and feedback space $\mathcal O$.
We say swapping $B$ for $B'$ is \emph{just panel change (JPC) on the method interface} if there exist Markov kernels
$Q:\mathcal W\mapsto \Delta\mathcal O$ and $Q':\mathcal W\mapsto\Delta\mathcal O$ such that, for every method
$\mathcal A$ and every horizon $T$,
the observable transcript laws factorize as
\begin{align}
\mathbb P^{\mathcal A}_{B}(dw_{1:T},do_{1:T})
&=
\prod_{t=1}^T \pi_t(dw_t\mid \mathcal I_{t-1}) \, Q(do_t\mid w_t), \label{eq:jpc_factor_B}\\
\mathbb P^{\mathcal A}_{B'}(dw_{1:T},do_{1:T})
&=
\prod_{t=1}^T \pi_t(dw_t\mid \mathcal I_{t-1}) \, Q'(do_t\mid w_t). \label{eq:jpc_factor_Bp}
\end{align}
In words: for every method, the swap preserves the observation/information structure and differs only in their kernels
$Q_{\mathrm{hum}}$ and $Q_{\mathrm{pers}}$.
\end{definition}

Definition~\ref{def:jpc_interface} is intentionally method-centric: it quantifies over \emph{all} (possibly adaptive)
methods and asks whether, from their perspective, the benchmark is an oracle channel that depends only on the current
submission $w_t$, with no additional side channels or method-dependent behavior.
Crucially, JPC does \emph{not} require $Q=Q'$ (scores need not match across humans and personas); it only requires that the
\emph{form} of the interaction is preserved and that the swap can be summarized entirely by replacing one artifact-to-score
kernel by another.
This is precisely the sense in which swapping evaluators should resemble a ``panel change'' rather than a change in the rules of
the game; the following Lemma \ref{lem:literal_implies_jpc} and \ref{lem:jpc_implies_literal_representation} formalize the mathematical equivalence.

\begin{lemma}[Literal panel change $\Rightarrow$ JPC]
\label{lem:literal_implies_jpc}
Assume Aggregate-only observation (AO), i.e., the method observes only $o_t\in\mathcal O$ each round.
Let $B=(P,I,\Gamma,L)$ and $B'=(P',I,\Gamma,L)$ differ by a literal panel change in the sense of
Definition~\ref{def:literal_panel_change}. Define the reduced-form kernels
\begin{align}
    Q_{P,I}(A\mid w)&:=\mathbb P\big(\Gamma(Z_1,\dots,Z_L)\in A\mid w\big), \notag
\\
Q_{P',I}(A\mid w)&:=\mathbb P\big(\Gamma(Z'_1,\dots,Z'_L)\in A\mid w\big), \notag
\end{align}
as in \eqref{eq:reducedform}. Then the swap $B\leftrightarrow B'$ is JPC in the sense of
Definition~\ref{def:jpc_interface}, with $Q:=Q_{P,I}$ and $Q':=Q_{P',I}$.
\end{lemma}

\begin{lemma}[JPC $\Rightarrow$ observational equivalence to a literal panel change]
\label{lem:jpc_implies_literal_representation}
Suppose two benchmarks $B$ and $B'$ satisfy JPC in the sense of Definition~\ref{def:jpc_interface}, i.e., there exist Markov
kernels $Q,Q':\mathcal W\mapsto\Delta\mathcal O$ such that for every method $\mathcal A$ and horizon $T$,
\begin{align}
    \mathbb P^{\mathcal A}_{B}(dw_{1:T},do_{1:T})
&=
\prod_{t=1}^T \pi_t(dw_t\mid \mathcal I_{t-1}) \, Q(do_t\mid w_t), \notag
\\
\mathbb P^{\mathcal A}_{B'}(dw_{1:T},do_{1:T})
&=
\prod_{t=1}^T \pi_t(dw_t\mid \mathcal I_{t-1}) \, Q'(do_t\mid w_t). \notag
\end{align}
Then there exists a pair of benchmarks $\overline B=(\overline P,\overline I,\overline\Gamma,\overline L)$ and
$\overline B'=(\overline P',\overline I,\overline\Gamma,\overline L)$ that differ by a \emph{literal panel change}
(i.e., only $\overline P$ is replaced by $\overline P'$), such that for every method $\mathcal A$ and horizon $T$,
\[
\mathbb P^{\mathcal A}_{\overline B}=\mathbb P^{\mathcal A}_{B},
\qquad
\mathbb P^{\mathcal A}_{\overline B'}=\mathbb P^{\mathcal A}_{B'}.
\]
In particular, JPC is \emph{exactly} the statement that the swap is indistinguishable on the method interface from a
literal panel change.
\end{lemma}

Lemma~\ref{lem:literal_implies_jpc} says that, if you \emph{literally} change only the panel distribution $P$ in the
panel--instrument--aggregation model, then the observable transcript laws change only by replacing the reduced-form kernel
$Q_{P,I}$ by $Q_{P',I}$, i.e., JPC holds.
Lemma~\ref{lem:jpc_implies_literal_representation} says the converse at the interface level: any JPC swap can be
realized \emph{exactly} as a literal panel change in a (possibly abstract) benchmark representation.
Thus Definition~\ref{def:jpc_interface} is mathematically equivalent to ``literal panel change'' \emph{as an interface-level
notion} (i.e., up to equality of transcript laws for every adaptive method).

The next subsection provides a protocol-level characterization of exactly when that interface property holds.

\subsection{Just panel change (JPC) iff (AO)+(MB)}
\label{ssec:main_thm_new}

\begin{lemma}[(JPC) $\iff$ (AO)+(MB)]
\label{lem:main_new}
Let $B_{\mathrm{hum}}$ and $B_{\mathrm{pers}}$ denote the human and persona benchmarking protocols,
with common artifact space $\mathcal W$ and feedback space $\mathcal O$.
The following are equivalent:
\begin{enumerate}[leftmargin=0.55cm,itemsep=2pt]
\item \textbf{(JPC).} The human $\leftrightarrow$ persona swap is just panel change (JPC) on the method interface
(Definition~\ref{def:jpc_interface}).
\item \textbf{(AO)+(MB).} The protocol satisfies aggregate-only observation (AO)
(Assumption~\ref{ass:AO}) and each benchmark is method-blind (MB)
(Assumption~\ref{ass:MB}).
\end{enumerate}
\end{lemma}


At a technical level, JPC is defined by a factorization of the transcript laws through an artifact-only kernel.
This makes the main theorem appear almost \textit{tautological}: if we assume conditions that guarantee exactly such a factorization,
we recover JPC.
The point of stating the equivalence is not to rebrand a definition, but to (i) express JPC in terms of two concrete and
auditable benchmark-hygiene requirements (what the benchmark reveals, and whether it is provenance-blind), and (ii) cleanly
separate the two ways persona benchmarking can fail as a substitute for field experiments: leakage of micro-level information
(violating AO) and method/provenance dependence (violating MB).

Lemma~\ref{lem:main_new} also highlights two distinct failure modes:

\paragraph{If (MB) fails, the benchmark is not a well-defined oracle environment.}
When (MB) fails, there is no single kernel $Q(\cdot\mid w)$ that governs the returned score across methods.
Equivalently, the evaluator's behavior depends on provenance/identity or on interaction history in a way that is not summarized by the submitted artifact.
In such a case, the benchmark cannot be treated as an artifact-only oracle channel, so ``panel change'' is not an identified description of the swap. See Appendix \ref{sec:practical_ab} for when MB can be justified.

\paragraph{If (AO) fails, the method interface changes (even if aggregate scores look the same).}
If the protocol reveals micro-level information (panel identities, raw votes, ordering, etc.), then two evaluator
implementations can induce the \emph{same} aggregate kernel on $\mathcal O$ but still be distinguishable
and exploitable by an adaptive method.
Hence the observable interaction is not characterized by the aggregate channel alone. See Appendix \ref{ssec:AO_counterexample} for when AO fails.

\section{Beyond validity: When does a persona panel constitute an \textit{useful} benchmark?}
\label{sec:usefulness}

As we discussed in section \ref{sec:identification}, (AO)+(MB) characterizes when and only when persona benchmarking is a \textit{valid} benchmark, satisfying (JPC).
But validity, which is an identification argument, is not necessarily equivalent to \textit{usefulness}. In this section, we formalize when and when not a persona-based LLM simulation can be a perfectly valid, but much less useful, benchmark compared to field experiments. We show that 1) dataset size is what really matters for a persona-based benchmark, or any benchmark, to be considered a useful benchmark, and 2) how to find such a required dataset size empirically.

As a starting point, recall that, under (AO)+(MB), the method interacts with the persona benchmark only through the reduced-form
kernel
\[
Q_{\mathrm{pers}}(\cdot\mid w)\quad \text{on}\quad \mathcal O,
\]
which returns an aggregate feedback draw $o\sim Q_{\mathrm{pers}}(\cdot\mid w)$ for each submitted artifact $w$.
Thus, we arrive at an important observation:
\begin{center}
    \textit{Whether a persona dataset will constitute an useful benchmark is a question about how informative the induced channel $Q_{\mathrm{pers}}$ is for \emph{distinguishing and optimizing} methods.}
\end{center}
Throughout this section, we will utilize this idea to derive the minimum persona dataset size for benchmarking.

\subsection{Usefulness and ``mountain-fog'' analogy}

Denote
\[
\mathcal W_0 \subseteq \mathcal W
\]
as \textit{the region of interest} for the artifacts (e.g., a tunable region of artifacts around a baseline artifact). Assume that  $\mathcal O=\mathbb R$ (a scalar score), define the \emph{benchmark landscape}
and \emph{benchmark noise} by
\begin{align}
    \mu_{\mathrm{pers}}(w) := \mathbb E[\,o \mid w\,], \quad
\sigma_{\mathrm{pers}}^2(w) := \mathrm{Var}(o \mid w),
\end{align}
where $o\sim Q_{\mathrm{pers}}(\cdot\mid w).$ 

A useful analogy to describe the usefulness of a benchmark is the ``mountain-fog'' metaphor.
The (unknown) benchmark landscape $w\mapsto \mu_{\mathrm{pers}}(w)$ is the \emph{mountain}: it assigns to each artifact the expected aggregate score returned by the benchmark.
The benchmark noise scale $w\mapsto \sigma_{\mathrm{pers}}(w)$ is the \emph{fog}.
This metaphor separates two different reasons a benchmark may not be useful. A benchmark can be a \emph{flat} mountain, meaning that $\mu_{\mathrm{pers}}(w)$ changes only slightly across meaningfully different artifacts; or it can be \emph{noisy} (thick fog), meaning that $\sigma_{\mathrm{pers}}(w)$ is large relative to the score differences the method is trying to detect.

\subsection{Formalization of usefulness: discriminability}
We formalize the ``mountain-fog'' idea by introducing a concept we call \textit{discriminability}. 
let $D_{\text{KL}}(\cdot,\cdot)$ be Kullback--Leibler divergence on $\mathcal O$. Also, suppose that we can define a metric $d_{\mathcal W}$ on $\mathcal W_0$. 

Let
$$
\mathcal{S}_r:=\left\{\left(w, w^{\prime}\right) \in \mathcal{W}_0 \times \mathcal{W}_0: d_{\mathcal{W}}\left(w, w^{\prime}\right) \geq r\right\} .
$$
where $r>0$ is a resolution parameter that formalizes ``minimal meaningful change'' in artifacts, which is often \emph{pre-specified ex ante}. The pair $(d_{\mathcal W},r)$ should be read as: \emph{``we only require the benchmark to separate
artifacts that differ by at least $r$ under $d_{\mathcal W}$.''} Choosing a smaller $r$ is a stricter requirement (it asks the benchmark to resolve finer changes),
and it can only make discriminability harder (the infimum ranges over a larger set, so $\kappa_Q$
can only decrease).
Thus $r$ should reflect the smallest change that is substantively meaningful \emph{for method development},
not the smallest change that can be expressed syntactically. In Section \ref{ssec:persona_implication}, we discuss the guidelines of figuring out $d_{\mathcal W}$ and $r$.

Fix a probability measure $\nu$ on $\mathcal{S}_r$. For $\left(w, w^{\prime}\right) \sim \nu$, define the random variable
$$
U:=D_{\mathrm{KL}}\left(Q_{\text {pers }}(\cdot \mid w)\| Q_{\text {pers }}\left(\cdot \mid w^{\prime}\right)\right).
$$
One conservative notion of discriminability is the quantity we call the $q$-quantile of $U$:
\begin{align}
\text{Quantile}_q(U):=\inf \{u: \nu(U \leq u) \geq q\} \label{eq:disc_def_usefulness}
\end{align}
Note that $\text {Quantile}_q(U)$ can be rewritten as 
$$
\sup _{\substack{E \subset \mathcal{S}_r \\ \nu(E) \geq 1-q}} \inf _{\left(w, w^{\prime}\right) \in E} D_{\mathrm{KL}}\left(Q(\cdot \mid w) \| Q\left(\cdot \mid w^{\prime}\right)\right)
$$
This implies that, for at least $1-q$ of the $r$-separated pairs (under $\nu$), the KL separation is at least $\text {Quantile}_q(U)$. If this is near zero, it means that distinct artifacts in $\mathcal W_0$ that are essentially
indistinguishable through the benchmark interface, and therefore, the benchmark provides little usable feedback for methods. 

In practice, however, the KL divergence is challenging to estimate empirically. The following Lemma \ref{lem:gaussian_d_usefulness} provides a nice simplification under the canonical assumption of homoscedastic Gaussian reduced-form kernels\footnote{The stadard quantitative Berry-Esseen results give a finite-sample error bound that decays at rate $O(1 / \sqrt{L})$ for bounded summands, making this Gaussian approximation increasingly accurate as the panel size grows \cite{berry1941accuracy}.}, resolving this challenge.

\begin{lemma}
\label{lem:gaussian_d_usefulness}
Assume $\mathcal O=\mathbb R$ and that a benchmark induces a Gaussian homoscedastic reduced-form kernel:
\begin{equation}
Q_{\mathrm{pers}}(\cdot\mid w)=\mathcal N\!\big(\mu_{\mathrm{pers}}(w),\,\sigma^2/L\big), \quad \sigma>0
\label{eq:gaussian_channel_assump_usefulness}
\end{equation}
where $L$ is the number of samples in the persona panel. Then for any $w,w'\in\mathcal W$,
\begin{align}
&D_{\mathrm{KL}}\!\big(Q_{\mathrm{pers}}(\cdot\mid w)\,\|\,Q_{\mathrm{pers}}(\cdot\mid w')\big) \notag
\\
\;&=\;
\frac{\big(\mu_{\mathrm{pers}}(w)-\mu_{\mathrm{pers}}(w')\big)^2}{2\sigma^2/L}=\,\frac{\Delta(w,w')^2}{2\sigma^2/L}.
\label{eq:gaussian_kl_snr}
\end{align}
\end{lemma}
The right-hand side of Equation \eqref{eq:gaussian_kl_snr} is closely related to the quantity we often call the pairwise signal-to-noise (SNR), which is defined as
\begin{equation}
\mathrm{SNR}(w,w')
\;:= \frac{\Delta(w,w')^2}{2\sigma^2} 
\notag
\end{equation}
This quantity is \textit{ empirically estimable}: it depends only on the first two moments of this kernel $Q$, namely the mean $\mu(w)=\mathbb{E}[o\mid w]$ and variance $\sigma^2(w)=\operatorname{Var}(o \mid w)$. 

\subsection{Sample complexity.}

We define \emph{per-sample $q$-robust discriminability}
of a benchmark as
\begin{align}
\kappa_Q(q):=\sup _{\substack{E \subset \mathcal{S}_r \\ \nu(E) \geq 1-q}} \inf _{\left(w, w^{\prime}\right) \in E} \operatorname{SNR}(w,w') 
\label{eq:disc_SNR}
\end{align}

$\kappa_Q(q)$
has a direct operational interpretation: it is the \emph{per-sample information} available
to distinguish two artifacts that differ by at least $r$ in the benchmark interface.

\begin{lemma}[Pairwise comparison sample complexity from discriminability]
\label{lem:disc_samplecomplexity_temp0} Make the same assumption as in Lemma \ref{lem:gaussian_d_usefulness}. \(L\) is the panel size used inside one benchmark call. Define \(\Delta(w,w'):=\mu_{\mathrm{pers}}(w)-\mu_{\mathrm{pers}}(w')\).

Then for \((W,W')\sim\nu\),
\begin{align}
    &\mathbb{P}\!\left(\{\widehat\mu_{\text{pers}}(W)\le \widehat\mu_{\text{pers}}(W'),\Delta\left(W, W^{\prime}\right)>0\}\right) \notag
    \\
    &\mathbb 
\;\le\;
q+\exp\!\left(-\frac{L}{2}\kappa_Q(q)\right).
\end{align}
In particular, choosing
\[
L \;\ge\; \frac{2}{\kappa_Q(q)}\log\frac{1}{\delta}
\]
gives \(\mathbb P(\widehat\mu_{\text{pers}}(W)\le \widehat\mu_{\text{pers}}(W'),\Delta\left(W, W^{\prime}\right)>0) \le q+\delta\).
\end{lemma}
Appendix \ref{app:exp-textbo} demonstrates an example of the sample complexity analysis for a
prompt-optimization-based self-improving AI system \cite{kang2025bayesian}. 


\subsection{Choice of $r$ and $d_{\mathcal W}$}
\label{ssec:persona_implication}

The definition of discriminability in~\eqref{eq:disc_SNR} depends on two user-specified
design choices: a metric $d_{\mathcal W}$ on the artifact space $\mathcal W_0$ and a resolution
threshold $r>0$. These are generally \emph{method- and task-specific} design parameters:
different method families explore different degrees of freedom in $\mathcal W$ and therefore induce
different natural notions of distance and resolution. Operationally, $d_{\mathcal W}$ and $r$ determine which pairs of artifacts the benchmark is required
to reliably distinguish, and therefore they determine the relevant sample complexity via
Lemma~\ref{lem:disc_samplecomplexity_temp0}.

Below are practical guidelines for selecting them in a way that is both interpretable and robust.

\paragraph{Tie $d_{\mathcal W}$ to the developer's degrees of freedom.}
A good default is to define $d_{\mathcal W}$ via the natural parameterization that methods actually tune.
If artifacts are produced by knobs $\theta\in\Theta$ through $w(\theta)$, and there is a natural distance
$d_\Theta$ on $\Theta$, one can induce a pseudo-metric on $\mathcal W$ by
\[
d_{\mathcal W}(w(\theta),w(\theta')):=d_\Theta(\theta,\theta').
\]
This makes $r$ interpretable as a \emph{step size in the space the method explores}.
Examples include the scaled Euclidean distance on continuous hyperparameters or the edit distance on a structured prompt template.

\paragraph{Choose $r$ as a \emph{minimal meaningful iteration} unit.}
In most benchmarking use cases, there is a natural notion of the smallest ``iteration'' a developer
expects to be worth distinguishing. The guiding principle is that $r$ should be \emph{large enough} that changes below $r$ are not worth spending benchmark
budget on, but \emph{small enough} that improvements developers actually seek fall above $r$"
\begin{itemize}[leftmargin=0.45cm,itemsep=0.5pt]
    \item Prompt/instruction tuning: $r$ can be ``one allowed edit'' under a pre-specified edit set
    (add/remove one constraint, add one example, modify one rubric item). Under token-level edit distance,
    this corresponds to a small fixed number of edits.
    \item Hyperparameter tuning: choose a scaled metric so that a standard ``one-step'' change has size $\approx 1$,
    then set $r=1$. For instance, scale each coordinate by a typical tuning increment.
    \item Model or policy variants: set $r$ to the smallest recipe change you would treat as a distinct method
    (e.g., one additional fine-tuning epoch, one dataset mixture adjustment above a threshold, a decoding rule change).
\end{itemize}

In short, $d_{\mathcal W}$ should encode \emph{meaningful artifact differences} (preferably behavioral and invariant
to cosmetic changes), while $r$ should encode the smallest change that developers intend to reliably resolve.
With these choices fixed, $\kappa_Q$ becomes an operationally estimable quantity, and
Lemma~\ref{lem:disc_samplecomplexity_temp0} translates it directly into the required persona data size for stable method comparison.

\subsection{Experiments}
As a proof-of-concept, Appendix~\ref{app:exp-textbo} demonstrates how our discriminability-based analysis can be used to calibrate the persona dataset size needed for reliable method comparison. Concretely, we apply the procedure to the persona-simulation ad benchmark used to evaluate \textsc{TextBO}, a prompt-optimization based self-improving AI system \cite{kang2025bayesian}, and compute the implied number of independent persona evaluations required to distinguish one-step prompt improvements at a chosen confidence level.


\section{Conclusion}

We characterized when LLM-persona panels can substitute for human field experiments as a \emph{benchmark interface} for method development. Our main result shows that persona vs.\ human evaluation is \emph{just panel change} from the method’s perspective if and only if two benchmark-hygiene conditions hold: (i) \emph{aggregate-only observation} (AO) and (ii) \emph{method-blind evaluation} (MB). When either fails, the benchmark can leak exploitable information or depend on provenance, breaking the interface-level equivalence.

We also separated \emph{validity} from \emph{usefulness}. A valid persona benchmark may still be less informative than field experimentation if the induced aggregate channel is too flat or noisy. Our discriminability $\kappa_Q$ (worst-case KL separation at resolution $r$) yields the corresponding budget scaling for reliable comparisons, on the order of $\kappa_Q^{-1}(q)\log(1/\delta)$. Thus, beyond enforcing (AO)+(MB), the practical requirement is a sufficient persona dataset.

\bibliography{example_paper}
\bibliographystyle{icml2026}

\appendix

\onecolumn

\section{Deferred theoretical discussions}\label{ssec:AppendixTheory}

\subsection{Proof of Lemma \ref{lem:literal_implies_jpc} and \ref{lem:jpc_implies_literal_representation} in Section \ref{sec:identification}}

\begin{proof}[Proof of Lemma \ref{lem:literal_implies_jpc}]
Fix any method $\mathcal A$ with submission kernels $\pi_t(\cdot\mid \mathcal I_{t-1})$ under (AO), and fix a round $t$.
Under benchmark $B=(P,I,\Gamma,L)$, conditional on the submitted artifact $w_t$, the benchmark generates
\[
p_{t,1},\dots,p_{t,L}\overset{i.i.d.}{\sim}P,\qquad
Z_{t,\ell}\sim I(\cdot\mid w_t,p_{t,\ell})\ \text{independently over }\ell,
\qquad
o_t=\Gamma(Z_{t,1},\dots,Z_{t,L}).
\]
By construction, given $w_t$ this sampling uses only fresh benchmark randomness (fresh panel draw and micro-responses)
and therefore does not depend on $\mathcal I_{t-1}$ nor on the identity of $\mathcal A$.
Hence for every measurable $A\subseteq\mathcal O$,
\[
\mathbb P_B^{\mathcal A}(o_t\in A\mid \mathcal I_{t-1},w_t)=Q_{P,I}(A\mid w_t)
\quad\text{a.s.}
\]
This is exactly the (MB)-type conditional independence statement with kernel $Q_{P,I}$.
Combining this with the fact that $w_t\sim \pi_t(\cdot\mid \mathcal I_{t-1})$ under (AO), the standard
sequential composition / chain rule for Markov kernels yields, for every horizon $T$,
\[
\mathbb P^{\mathcal A}_{B}(dw_{1:T},do_{1:T})
=
\prod_{t=1}^T \pi_t(dw_t\mid \mathcal I_{t-1}) \, Q_{P,I}(do_t\mid w_t).
\]
The same argument for $B'=(P',I,\Gamma,L)$ gives
\[
\mathbb P^{\mathcal A}_{B'}(dw_{1:T},do_{1:T})
=
\prod_{t=1}^T \pi_t(dw_t\mid \mathcal I_{t-1}) \, Q_{P',I}(do_t\mid w_t).
\]
This is precisely Definition~\ref{def:jpc_interface} with $Q:=Q_{P,I}$ and $Q':=Q_{P',I}$.
\end{proof}

\begin{proof}[Proof of Lemma \ref{lem:jpc_implies_literal_representation}]
We construct an \emph{auxiliary} panel--instrument--aggregation representation that reproduces the same transcript laws.
This construction is not meant to mirror the internal structure of the original human/persona protocols. Let
\begin{align}
    \overline{\mathcal P}&:=\{0,1\},
\overline{\mathcal Z}:=\mathcal O,\
\overline L:=1, \notag
\\
\overline\Gamma&:\overline{\mathcal Z}\to\mathcal O\ \text{be the identity map } \overline\Gamma(o)=o. \notag
\end{align}
Define a \emph{single} micro-instrument $\overline I$ on $\overline{\mathcal Z}=\mathcal O$ by
\[
\overline I(\cdot\mid w,0):=Q(\cdot\mid w),
\qquad
\overline I(\cdot\mid w,1):=Q'(\cdot\mid w).
\]
Now define the two panel distributions
\[
\overline P:=\delta_0,\qquad \overline P':=\delta_1.
\]
Then $\overline B:=(\overline P,\overline I,\overline\Gamma,\overline L)$ and
$\overline B':=(\overline P',\overline I,\overline\Gamma,\overline L)$ differ \emph{only} in the panel distribution
($\delta_0$ versus $\delta_1$), hence are a literal panel change.

Moreover, under $\overline B$, each benchmark call on $w$ samples $p=0$ a.s., then outputs $o\sim \overline I(\cdot\mid w,0)=Q(\cdot\mid w)$.
Thus the reduced-form kernel of $\overline B$ is exactly $Q$. Similarly, the reduced-form kernel of $\overline B'$ is exactly $Q'$.
Therefore, for every method $\mathcal A$ and horizon $T$, the transcript laws under $\overline B$ and $\overline B'$ satisfy
the same factorizations as in the displayed JPC equations, which implies
\begin{align}
   &\mathbb P^{\mathcal A}_{\overline B}(dw_{1:T},do_{1:T}) \notag
\\
&=
\prod_{t=1}^T \pi_t(dw_t\mid \mathcal I_{t-1})\,Q(do_t\mid w_t)
=
\mathbb P^{\mathcal A}_{B}(dw_{1:T},do_{1:T}), \notag
\end{align}
and likewise $\mathbb P^{\mathcal A}_{\overline B'}=\mathbb P^{\mathcal A}_{B'}$.
\end{proof}

\subsection{Proof of Lemma \ref{lem:main_new} in Section \ref{sec:identification}}
\label{ssec:mainthmproof}

\begin{proof}[Proof of Lemma~\ref{lem:main_new}]
We prove $(2)\Rightarrow(1)$ and $(1)\Rightarrow(2)$.

\smallskip
\noindent\textbf{$(2)\Rightarrow(1)$: (AO)+(MB) imply JPC.}
Assume (AO) holds and each benchmark is method-blind (MB) in the sense of Assumptions~\ref{ass:AO}--\ref{ass:MB}.
Apply Lemma~\ref{lem:factorization_from_MB} to the human benchmark $B_{\mathrm{hum}}$ to obtain a kernel
$Q_{\mathrm{hum}}:\mathcal W\mapsto \Delta\mathcal O$ such that for every method $\mathcal A$ and horizon $T$,
\[
\mathbb P^{\mathcal A}_{B_{\mathrm{hum}}}(dw_{1:T},do_{1:T})
=
\prod_{t=1}^T \pi_t(dw_t\mid \mathcal I_{t-1}) \, Q_{\mathrm{hum}}(do_t\mid w_t).
\]
Likewise apply Lemma~\ref{lem:factorization_from_MB} to $B_{\mathrm{pers}}$ to obtain
$Q_{\mathrm{pers}}$ with
\[
\mathbb P^{\mathcal A}_{B_{\mathrm{pers}}}(dw_{1:T},do_{1:T})
=
\prod_{t=1}^T \pi_t(dw_t\mid \mathcal I_{t-1}) \, Q_{\mathrm{pers}}(do_t\mid w_t).
\]
Setting $Q:=Q_{\mathrm{hum}}$ and $Q':=Q_{\mathrm{pers}}$ verifies Definition~\ref{def:jpc_interface}.
Hence the swap is JPC on the method interface.

\smallskip
\noindent\textbf{$(1)\Rightarrow(2)$: JPC implies (AO)+(MB).}
Assume the swap is JPC in the sense of Definition~\ref{def:jpc_interface}.
By that definition, the transcript laws of both benchmarks factorize for every method $\mathcal A$ and every horizon $T$
through some kernels $Q$ and $Q'$.

First, (AO) is the observation/information structure assumed in Definition~\ref{def:jpc_interface}
(the method's interaction is summarized by the aggregate-history $\mathcal I_{t-1}$ and transcript $(w_t,o_t)$).
Second, applying Lemma~\ref{lem:MB_from_factorization} to $B_{\mathrm{hum}}$ and $B_{\mathrm{pers}}$ shows that each benchmark
satisfies (MB) with kernels $Q_{\mathrm{hum}}=Q$ and $Q_{\mathrm{pers}}=Q'$ respectively.

Therefore, (AO)+(MB) hold.
\end{proof}

\begin{lemma}[Transcript factorization under (AO)+(MB)]
\label{lem:factorization_from_MB}
Fix a benchmark $B$ with artifact space $W$ and feedback space $O$.
Assume (AO) and (MB) hold for $B$, i.e., there exists a Markov kernel
$Q_B: W \to \Delta(O)$ such that for every method $A$, every round $t$,
and every measurable $A \subseteq O$,

\begin{equation}
\mathbb{P}^A_B(o_t \in A \mid I_{t-1}, w_t) = Q_B(A \mid w_t) \quad \text{a.s.} 
\label{eq:MB_again}
\end{equation}
Then for every method $A$ (with submission kernels $\pi_t(\cdot \mid I_{t-1})$ under (AO))
and every horizon $T$,
\begin{equation}
\mathbb{P}^A_B(dw_{1:T}, do_{1:T})
= \prod_{t=1}^T \pi_t(dw_t \mid I_{t-1})\, Q_B(do_t \mid w_t),
\label{eq:factorization_B}
\end{equation}
\end{lemma}

\begin{proof}[Proof of Lemma \ref{lem:factorization_from_MB}]
Fix a method $A$. Under (AO), at each round $t$ the method chooses
$w_t \sim \pi_t(\cdot \mid I_{t-1})$.
By (MB) in (12), conditional on $(I_{t-1}, w_t)$ the benchmark draw satisfies
$o_t \sim Q_B(\cdot \mid w_t)$.
Therefore, conditional on $I_{t-1}$, the pair $(w_t,o_t)$ is generated by
\[
w_t \sim \pi_t(\cdot \mid I_{t-1}), \qquad o_t \sim Q_B(\cdot \mid w_t).
\]
Iterating over $t=1,\dots,T$ and applying the standard chain rule / sequential
composition for Markov kernels yields Equation \eqref{eq:factorization_B}.
\end{proof}

\begin{lemma}[Factorization implies (MB)]
\label{lem:MB_from_factorization}
Fix a benchmark $B$.
Suppose that under (AO), there exists a Markov kernel $Q_B:\mathcal W\mapsto\Delta\mathcal O$
such that for every method $\mathcal A$ and every horizon $T$,
\[
\mathbb P^{\mathcal A}_B(dw_{1:T},do_{1:T})
=
\prod_{t=1}^T \pi_t(dw_t\mid \mathcal I_{t-1}) \, Q_B(do_t\mid w_t).
\]
Then $B$ satisfies (MB) in the sense of Assumption~\ref{ass:MB}; namely, for every method $\mathcal A$, every $t$,
and measurable $A\subseteq \mathcal O$,
\[
\mathbb P^{\mathcal A}_B(o_t\in A \mid I_{t-1}, w_t) = Q_B(A\mid w_t)
\quad\text{a.s.}
\]
In particular, by tower property, this implies $\mathbb{P}_B^A\left(o_t \in\right. \left.A \mid H_{t-1}, w_t\right)=Q_B\left(A \mid w_t\right)\;\text{a.s.}$ as well.
\end{lemma}

\begin{proof}[Proof of Lemma \ref{lem:MB_from_factorization}]
Fix $\mathcal A,t$ and a measurable $A\subseteq \mathcal O$.

Under the assumed factorization, conditional on $\mathcal I_{t-1}$ and $w_t$, the draw $o_t$ is generated by
$Q_B(\cdot\mid w_t)$, so
\[
\mathbb P^{\mathcal A}_B(o_t\in A \mid \mathcal I_{t-1}, w_t) = Q_B(A\mid w_t)\quad\text{a.s.}
\]
Now apply the tower property conditioning down to $(H_{t-1},w_t)$:
\begin{align*}
\mathbb P^{\mathcal A}_B(o_t\in A \mid H_{t-1}, w_t)
&=
\mathbb E^{\mathcal A}_B\!\left[
\mathbb P^{\mathcal A}_B(o_t\in A \mid \mathcal I_{t-1}, w_t)
\;\middle|\; H_{t-1}, w_t\right]\\
&=
\mathbb E^{\mathcal A}_B\!\left[ Q_B(A\mid w_t)\mid H_{t-1}, w_t\right]\\
&= Q_B(A\mid w_t),
\end{align*}
since $Q_B(A\mid w_t)$ is $\sigma(w_t)$-measurable.
This is exactly (MB).
\end{proof}

\subsection{Proofs in Section \ref{sec:usefulness}}

\begin{proof} [Proof of Lemma \ref{lem:gaussian_d_usefulness}.]
Let $P_x$ denote the density of $\mathcal N(x,\sigma^2)$:
\[
p_x(o)=\frac{1}{\sqrt{2\pi\sigma^2}}\exp\!\left(-\frac{(o-x)^2}{2\sigma^2}\right).
\]
By definition,
\[
d(x,y)=\mathbb E_{O\sim \mathcal N(x,\sigma^2)}\!\left[\log\frac{p_x(O)}{p_y(O)}\right].
\]
Compute the log-likelihood ratio:
\begin{align*}
\log\frac{p_x(O)}{p_y(O)}
&=
-\frac{(O-x)^2}{2\sigma^2}+\frac{(O-y)^2}{2\sigma^2}
\\
&=
\frac{(O-y)^2-(O-x)^2}{2\sigma^2}
\\
&=
\frac{(x-y)\big(2O-x-y\big)}{2\sigma^2}.
\end{align*}
Taking expectation under $O\sim\mathcal N(x,\sigma^2)$ yields
\begin{align}
   d(x,y)  &=\frac{x-y}{2\sigma^2}\,\mathbb E\big[2O-x-y\big]\notag\\
    &
=\frac{x-y}{2\sigma^2}\,(2x-x-y)
=\frac{(x-y)^2}{2\sigma^2},\notag
\end{align}
\end{proof}

\begin{proof}[Proof of Lemma \ref{lem:disc_samplecomplexity_temp0}]
Fix any \((w,w')\in\mathcal S_r\) and abbreviate \(\Delta:=\Delta(w,w')=\mu_{\mathrm{pers}}(w)-\mu_{\mathrm{pers}}(w')\).
Under the Gaussian reduced-form assumption (as in Lemma~\ref{lem:gaussian_d_usefulness}),
a single benchmark call with panel size \(L\) returns
\[
\widehat\mu_{\mathrm{pers}}(w)\sim \mathcal N\!\left(\mu_{\mathrm{pers}}(w),\frac{\sigma^2}{L}\right),
\qquad
\widehat\mu_{\mathrm{pers}}(w')\sim \mathcal N\!\left(\mu_{\mathrm{pers}}(w'),\frac{\sigma^2}{L}\right),
\]
independently (given \(w,w'\)). Hence
\[
\widehat\Delta \;:=\; \widehat\mu_{\mathrm{pers}}(w)-\widehat\mu_{\mathrm{pers}}(w')
\sim \mathcal N\!\left(\Delta,\frac{2\sigma^2}{L}\right).
\]
If \(\Delta\le 0\), then
\(\mathbb P(\widehat\mu_{\mathrm{pers}}(w)\le \widehat\mu_{\mathrm{pers}}(w'),\,\Delta>0\mid w,w')=0\).
Assume \(\Delta>0\). Then the event
\(\{\widehat\mu_{\mathrm{pers}}(w)\le \widehat\mu_{\mathrm{pers}}(w')\}\) is \(\{\widehat\Delta\le 0\}\).
For any \(t>0\), Markov's inequality gives
\[
\mathbb P(\widehat\Delta\le 0)
=
\mathbb P\!\left(e^{-t\widehat\Delta}\ge 1\right)
\le
\mathbb E\!\left[e^{-t\widehat\Delta}\right]
=
\exp\!\left(-t\Delta+\frac{t^2}{2}\cdot\frac{2\sigma^2}{L}\right)
=
\exp\!\left(-t\Delta+\frac{t^2\sigma^2}{L}\right).
\]
Optimizing over \(t\) yields \(t^\star=\frac{L\Delta}{2\sigma^2}\), hence
\[
\mathbb P(\widehat\Delta\le 0)
\le
\exp\!\left(-\frac{L\Delta^2}{4\sigma^2}\right)
=
\exp\!\left(-\frac{L}{2}\cdot \frac{\Delta^2}{2\sigma^2}\right).
\]
Recalling \(\mathrm{SNR}(w,w'):=\frac{\Delta(w,w')^2}{2\sigma^2}\), we obtain the conditional bound
\[
\mathbb P\!\left(\widehat\mu_{\mathrm{pers}}(w)\le \widehat\mu_{\mathrm{pers}}(w'),\,\Delta(w,w')>0 \,\middle|\, w,w'\right)
\le
\exp\!\left(-\frac{L}{2}\,\mathrm{SNR}(w,w')\right).
\]

Now draw \((W,W')\sim\nu\) and let \(V:=\mathrm{SNR}(W,W')\). Taking expectations,
\[
\mathbb P\!\left(\widehat\mu_{\mathrm{pers}}(W)\le \widehat\mu_{\mathrm{pers}}(W'),\,\Delta(W,W')>0\right)
\le
\mathbb E\!\left[\exp\!\left(-\frac{L}{2}V\right)\right].
\]
Let \(\mathcal B:=\{V<\kappa_Q(q)\}\) and \(\mathcal G:=\{V\ge \kappa_Q(q)\}\).
By the definition of \(\kappa_Q(q)\) as the \(q\)-quantile / \(q\)-robust discriminability of \(V\)
(cf.~\eqref{eq:disc_SNR} and its quantile-equivalent form), we have \(\nu(\mathcal B)\le q\).
Therefore,
\begin{align*}
\mathbb E\!\left[e^{-\frac{L}{2}V}\right]
&=
\mathbb E\!\left[e^{-\frac{L}{2}V}\mathbf 1_{\mathcal B}\right]
+
\mathbb E\!\left[e^{-\frac{L}{2}V}\mathbf 1_{\mathcal G}\right]\\
&\le
\nu(\mathcal B)\cdot 1
+
e^{-\frac{L}{2}\kappa_Q(q)}\nu(\mathcal G)
\;\le\;
q+e^{-\frac{L}{2}\kappa_Q(q)}.
\end{align*}
This proves the claimed bound.

Finally, if \(L \ge \frac{2}{\kappa_Q(q)}\log\frac{1}{\delta}\), then
\(\exp(-\frac{L}{2}\kappa_Q(q))\le \delta\), so
\[
\mathbb P\!\left(\widehat\mu_{\mathrm{pers}}(W)\le \widehat\mu_{\mathrm{pers}}(W'),\,\Delta(W,W')>0\right)
\le q+\delta.
\]
\end{proof}

\newpage

\section{Extended discussions}

\subsection{AO and JPC - a counterexample}
\label{ssec:AO_counterexample}

The equivalence in Lemma~\ref{lem:main_new} highlights that (AO) is not merely technical:
if micro-level information leaks (raw votes, rater identities, ordering, etc.), then two
evaluation pipelines can induce the \emph{same} aggregate kernel $Q(\cdot\mid w)$ yet still be
distinguishable (and exploitable) by an adaptive method. The following minimal construction
makes this necessity direction concrete.

\begin{proposition}[Violating (AO) can break JPC even when aggregate kernels match]
\label{prop:AO_counterexample}
There exist two benchmarks $B$ and $B'$ that induce the \emph{same} reduced-form kernel
$Q(\cdot\mid w)$ on the aggregate feedback space $\mathcal O$, but such that if the benchmark
reveals the raw vote vector $(Z_1,\dots,Z_L)$ to the method (violating (AO)), then there is an
adaptive method $\mathcal A$ whose aggregate transcript law
$\mathbb P^{\mathcal A}(w_{1:T},o_{1:T})$ differs between $B$ and $B'$.
Consequently, the swap cannot be ``just panel change'' on the method interface as in
Definition~\ref{def:jpc_interface}.
\end{proposition}

\begin{proof}
We specify two benchmarks and an adaptive method.

\paragraph{Spaces and aggregation.}
Let $\mathcal W=\{0,1\}$, $\mathcal Z=\{0,1\}$, $\mathcal O=\{0,1\}$, and take panel size $L=2$.
Let the aggregation map be the XOR (disagreement) statistic
\[
\Gamma(z_1,z_2):=z_1\oplus z_2\in\{0,1\}.
\]
Thus the aggregate output is $o=1$ iff the two individual votes disagree.

\paragraph{Two micro-instruments with identical aggregate kernels.}
Let $p_0:=0.1$ and $p_1:=0.4$. Define benchmark $B$ so that, for each submitted artifact $w\in\{0,1\}$,
the two micro-votes are independent Bernoulli draws
\[
Z_1,Z_2 \overset{i.i.d.}{\sim} \mathrm{Bernoulli}(p_w).
\]
Define benchmark $B'$ identically except that each vote is \emph{flipped} in distribution:
\[
Z'_1,Z'_2 \overset{i.i.d.}{\sim} \mathrm{Bernoulli}(1-p_w).
\]
(Panel sampling is irrelevant here; one can take a degenerate panel distribution and absorb everything
into the micro-instrument.)

Now compute the induced reduced-form kernels on $\mathcal O$. Under $B$,
\[
\mathbb P_B(o=1\mid w)
=\mathbb P_B(Z_1\neq Z_2\mid w)
=2p_w(1-p_w).
\]
Under $B'$,
\[
\mathbb P_{B'}(o=1\mid w)
=\mathbb P_{B'}(Z'_1\neq Z'_2\mid w)
=2(1-p_w)p_w
=2p_w(1-p_w).
\]
Hence the aggregate channels coincide exactly:
\[
Q_B(\cdot\mid w)=Q_{B'}(\cdot\mid w)\qquad\forall\,w\in\mathcal W.
\]
Concretely, $Q(o=1\mid 0)=2(0.1)(0.9)=0.18$ and $Q(o=1\mid 1)=2(0.4)(0.6)=0.48$ for both benchmarks.

\paragraph{AO violation (raw-vote leakage) and an adaptive distinguisher.}
Now suppose (AO) is violated and the benchmark releases the raw vote vector
$(Z_1,Z_2)$ (or $(Z'_1,Z'_2)$) to the method in addition to the aggregate $o$.
Consider the following horizon-$T=2$ adaptive method $\mathcal A$:
\begin{itemize}[leftmargin=0.45cm,itemsep=2pt]
\item Round 1: submit $w_1=0$.
\item Observe the \emph{first} raw vote and set $w_2 := Z_{1,1}$ (i.e., the round-1 vote of evaluator 1).
\end{itemize}
This is a valid adaptive strategy under the leaked interface, but it is \emph{not} measurable with respect
to the aggregate-only history $(w_1,o_1)$.

Under benchmark $B$, $\mathbb P(w_2=1)=\mathbb P(Z_{1,1}=1\mid w_1=0)=p_0=0.1$.
Under benchmark $B'$, $\mathbb P(w_2=1)=\mathbb P(Z'_{1,1}=1\mid w_1=0)=1-p_0=0.9$.
Therefore the distribution of the second-round submission $w_2$ differs between the two benchmarks even though
the aggregate kernel $Q(\cdot\mid w)$ is identical.

Because $o_2\sim Q(\cdot\mid w_2)$ in both benchmarks, this also induces a difference in the aggregate outcome at round 2:
\begin{align*}
\mathbb P_B(o_2=1)
&=\mathbb P_B(w_2=1)\,Q(o=1\mid 1)+\mathbb P_B(w_2=0)\,Q(o=1\mid 0) \\
&=0.1\cdot 0.48 + 0.9\cdot 0.18
=0.21,
\\
\mathbb P_{B'}(o_2=1)
&=0.9\cdot 0.48 + 0.1\cdot 0.18
=0.45.
\end{align*}
Thus the aggregate transcript laws $\mathbb P^{\mathcal A}_B(w_{1:2},o_{1:2})$ and
$\mathbb P^{\mathcal A}_{B'}(w_{1:2},o_{1:2})$ differ.

\paragraph{Conclusion.}
If the swap $B\leftrightarrow B'$ were JPC on the aggregate-only method interface in the sense of
Definition~\ref{def:jpc_interface}, then (since $Q_B=Q_{B'}$) the factorization would force
\emph{every} method's aggregate transcript law to agree under $B$ and $B'$.
The above $\mathcal A$ contradicts this as soon as raw votes are released.
Therefore leaking micro-level information (violating (AO)) can break JPC even when the aggregate kernels match.
\end{proof}

The same phenomenon occurs if the benchmark releases panel identities (or stable rater IDs):
even if the aggregate score distribution $Q(\cdot\mid w)$ is unchanged, the extra identifier acts as a
side channel that an adaptive method can condition on, producing different submission sequences and
hence different aggregate transcripts. This is exactly why (AO) is a \emph{benchmark-hygiene} requirement:
it rules out side channels through which the method can distinguish evaluator implementations that are
otherwise identical at the aggregate level.

\subsection{Practical realism of method-blind evaluation (MB)}
\label{sec:practical_ab}

Assumption~\ref{ass:MB} is a \emph{protocol} requirement: conditional on the submitted artifact $w$, the
distribution of the returned aggregate feedback must not depend on who submitted $w$ or how it was
produced. This is intentionally stronger than what one gets ``by default'' in many human-judgment
settings, because there is extensive evidence that evaluations can shift when raters are exposed to
provenance cues (e.g., explicit labels such as ``AI-generated'' vs.\ ``human-written'') or other contextual
metadata beyond the content being judged \cite{bogert2022human,osborne2025me,levitt2011hawthorne,dequidt2018demand}.
In our framework, such provenance dependence is precisely an (MB) violation: two identical artifacts $w$
can induce different score distributions if the evaluator observes additional information correlated with
the producing method.\footnote{If provenance cues are literally part of the submitted artifact (e.g.,
the text contains ``as an AI language model''), then any resulting penalty is \emph{not} an (MB) violation,
because it is a function of $w$ itself. The problematic case for (MB) is when the evaluator is shown
extra metadata (model name, submitter identity, method label, timestamp-based context, etc.) that is
\emph{not} a function of $w$.}

\paragraph{How realistic is (MB) in practice?}
(MB) is best viewed as a \emph{design target} that is often approximately achievable, but not automatic.
It is most realistic in settings where the outcome is behavioral and passively recorded (classic online
A/B tests), since users are typically not told which method produced what they see and outcomes like
click-through or conversion are not direct subjective judgments \cite{kohavi2020trustworthy}.
By contrast, in explicit rating / preference-judgment pipelines (crowd or expert), (MB) is fragile:
even minimal provenance cues (labels, branding, prior beliefs about model quality, expectations about
``humans vs.\ AI'') can create systematic shifts in ratings for the \emph{same} artifact, violating (MB).
This fragility is also consistent with the motivation for anonymous pairwise evaluation interfaces used
in large-scale leaderboards \cite{chiang2024chatbotarena} and with documented ordering/labeling effects in
LLM-judge-style protocols \cite{dominguezolmedo2024questioning}.

The key takeaway is that (MB) should not be read as an empirical claim about human invariance; it is a
\emph{benchmark hygiene constraint} that must be enforced (or at least audited) by the benchmark
implementation.

\subsection{Extension to heteroscedastic Gaussian reduced-form kernels}
\label{ssec:heteroscedastic_gaussian_usefulness}

Lemma~\ref{lem:gaussian_d_usefulness} assumes a homoscedastic Gaussian reduced-form kernel, i.e.,
$\mathrm{Var}(o\mid w)$ is constant across artifacts.
In many benchmarks, however, the aggregate score variance depends on the submitted artifact:
some artifacts elicit highly consistent micro-responses (low noise), while others are ambiguous or polarizing
(high noise). This motivates the more general \emph{heteroscedastic} Gaussian model
\begin{equation}
Q_{\mathrm{pers}}(\cdot\mid w)=\mathcal N\!\big(\mu_{\mathrm{pers}}(w),\,\sigma_{\mathrm{pers}}^2(w)\big),
\qquad \sigma_{\mathrm{pers}}(w)>0.
\label{eq:hetero_gaussian_channel_assump_usefulness}
\end{equation}
(Here $\sigma_{\mathrm{pers}}^2(w)$ denotes the variance of the \emph{aggregate} feedback draw $o\sim Q_{\mathrm{pers}}(\cdot\mid w)$.
Any micro-level panel size has already been absorbed into this reduced-form variance.)

\begin{lemma}[KL divergence under heteroscedastic Gaussians]
\label{lem:hetero_gaussian_kl_usefulness}
Assume $\mathcal O=\mathbb R$ and~\eqref{eq:hetero_gaussian_channel_assump_usefulness}. Then for any $w,w'\in\mathcal W$,
\begin{align}
&D_{\mathrm{KL}}\!\big(Q_{\mathrm{pers}}(\cdot\mid w)\,\big\|\,Q_{\mathrm{pers}}(\cdot\mid w')\big)
\notag\\
&=\frac{1}{2}\Bigg[
\log\frac{\sigma_{\mathrm{pers}}^2(w')}{\sigma_{\mathrm{pers}}^2(w)}
+\frac{\sigma_{\mathrm{pers}}^2(w)}{\sigma_{\mathrm{pers}}^2(w')}
-1
\Bigg]
\;+\;
\frac{\big(\mu_{\mathrm{pers}}(w)-\mu_{\mathrm{pers}}(w')\big)^2}{2\,\sigma_{\mathrm{pers}}^2(w')}.
\label{eq:hetero_gaussian_kl_closed_form}
\end{align}

\end{lemma}

Equation~\eqref{eq:hetero_gaussian_kl_closed_form} highlights two separable sources of
distributional distinguishability:
(i) a \emph{variance-mismatch} term (the bracketed expression), which is zero iff
$\sigma_{\mathrm{pers}}^2(w)=\sigma_{\mathrm{pers}}^2(w')$, and
(ii) a \emph{mean-separation} term, which scales the squared mean gap by $1/\sigma_{\mathrm{pers}}^2(w')$.
In contrast to the homoscedastic case, the KL divergence is generally \emph{asymmetric}:
$D_{\mathrm{KL}}(Q(\cdot\mid w)\|Q(\cdot\mid w'))\neq D_{\mathrm{KL}}(Q(\cdot\mid w')\|Q(\cdot\mid w))$ when
$\sigma_{\mathrm{pers}}^2(w)\neq\sigma_{\mathrm{pers}}^2(w')$.

\begin{proof}[Proof of Lemma~\ref{lem:hetero_gaussian_kl_usefulness}]
Fix $w,w'\in\mathcal W$ and abbreviate
\[
\mu:=\mu_{\mathrm{pers}}(w),\quad \mu':=\mu_{\mathrm{pers}}(w'),\quad 
\sigma^2:=\sigma_{\mathrm{pers}}^2(w),\quad \tau^2:=\sigma_{\mathrm{pers}}^2(w').
\]
Let $P:=\mathcal N(\mu,\sigma^2)$ and $Q:=\mathcal N(\mu',\tau^2)$ with Lebesgue densities
\[
p(o)=\frac{1}{\sqrt{2\pi\sigma^2}}\exp\!\left(-\frac{(o-\mu)^2}{2\sigma^2}\right),
\qquad
q(o)=\frac{1}{\sqrt{2\pi\tau^2}}\exp\!\left(-\frac{(o-\mu')^2}{2\tau^2}\right).
\]
By definition,
\[
D_{\mathrm{KL}}(P\|Q)=\mathbb E_{O\sim P}\!\left[\log\frac{p(O)}{q(O)}\right].
\]
Compute the log-likelihood ratio:
\begin{align*}
\log\frac{p(O)}{q(O)}
&=
\left(-\frac{1}{2}\log(2\pi\sigma^2)-\frac{(O-\mu)^2}{2\sigma^2}\right)
-\left(-\frac{1}{2}\log(2\pi\tau^2)-\frac{(O-\mu')^2}{2\tau^2}\right)\\
&=
\frac{1}{2}\log\frac{\tau^2}{\sigma^2}
+\frac{(O-\mu')^2}{2\tau^2}
-\frac{(O-\mu)^2}{2\sigma^2}.
\end{align*}
Taking expectations under $O\sim\mathcal N(\mu,\sigma^2)$, we use the identities
\[
\mathbb E[(O-\mu)^2]=\sigma^2,
\qquad
\mathbb E[(O-\mu')^2]=\mathrm{Var}(O)+(\mathbb E[O]-\mu')^2=\sigma^2+(\mu-\mu')^2.
\]
Hence
\begin{align*}
D_{\mathrm{KL}}(P\|Q)
&=
\frac{1}{2}\log\frac{\tau^2}{\sigma^2}
+\frac{1}{2\tau^2}\Big(\sigma^2+(\mu-\mu')^2\Big)
-\frac{1}{2\sigma^2}\sigma^2\\
&=
\frac{1}{2}\left[
\log\frac{\tau^2}{\sigma^2}+\frac{\sigma^2}{\tau^2}-1
\right]
+\frac{(\mu-\mu')^2}{2\tau^2}.
\end{align*}
Substituting back $\sigma^2=\sigma_{\mathrm{pers}}^2(w)$ and $\tau^2=\sigma_{\mathrm{pers}}^2(w')$
gives~\eqref{eq:hetero_gaussian_kl_closed_form}.

\end{proof}

\newpage

\section{Experiment: Discriminability Calibration for the Persona-Based Ad Benchmark}
\label{app:exp-textbo}

This experiment provides a proof-of-concept discriminability audit for persona-simulation evaluation environment
used in demonstrating the effectiveness of \textsc{TextBO} \cite{kang2025bayesian}, a prompt optimization based self-improving AI method. Specifically, we focus on the ad optimization experimentation proposed in this paper. 

In the ad optimization experiment, \textsc{TextBO} iteratively improves the \textit{prompt for ad image generation} that is fed to an image model to produce the ad creative, using evaluation feedback to decide which prompt edits to keep. Concretely, it runs an automated prompt improvement loop: \emph{write a prompt} \(\rightarrow\) \emph{generate an ad image} \(\rightarrow\) \emph{evaluate it on a target audience} \(\rightarrow\) \emph{reflection} \(\rightarrow\) \emph{edit the prompt and repeat}.

As in \citet{kang2025bayesian}, we consider eight synthetic ad campaign scenarios that cover a diverse range of products across distinct categories, each defined by a creative brief that outlines the strategic and creative direction (see Web Appendix C.1 of \citet{kang2025bayesian} for the full creative briefs):
\begin{itemize}[leftmargin=0.5cm,itemsep=0pt]
    \item Scenario 1: ``GreenBite,'' a new plant-based burger patty.
    \item Scenario 2: ``AuraSonics X1,'' high-end, noise-canceling wireless earbuds.
    \item Scenario 3: ``Odyssey E-SUV,'' a new all-electric family SUV.
    \item Scenario 4: ``Oasis Eco-Lodge,'' a secluded, luxury resort with beautiful natural surroundings.
    \item Scenario 5: ``Momentum,'' a mobile-first banking app for freelancers and the gig economy.
    \item Scenario 6: ``MindGarden,'' a subscription-based meditation and mindfulness app.
    \item Scenario 7: ``Aeterno,'' a classic, automatic Swiss-made wristwatch with a heritage design.
    \item Scenario 8: ``SyncFlow,'' a project management and collaboration software platform for remote teams.
\end{itemize}

For evaluation, \textsc{TextBO} utilizes Twin-2k-500 persona dataset \cite{toubia2025database}. 
For a fixed scenario and prompt, TextBO evaluates an ad image by:
(i) sampling 200 personas from the Twin-2k-500 \emph{training} split,
(ii) conditioning a multimodal LLM judge on each persona’s survey answers, requesting an effectiveness rating on the 1--5 scale,
(iii) converting the judge’s log-probabilities over \{1,2,3,4,5\} into a real-valued expected score per persona,
and (iv) averaging over the 200 personas to produce a single scalar score.
This single scalar is the only feedback used by the optimization method. We use Gemini 2.5 Flash with the same meta-prompt used in \citet{kang2025bayesian}, which is provided in Figure \ref{fig:persona_prompt}.

\begin{figure}[!ht]
\centering
\begin{tcolorbox}[
    colback=promptbg!100!white,
    sharp corners=south, 
    boxrule=0.25mm, 
    fonttitle=\bfseries, 
    title={\textbf{Persona prompt for simulating ad effectiveness.}}, 
    width=\textwidth, 
    enhanced,
    drop shadow
    ]
\scriptsize
\setlength{\parskip}{0pt}
\setlength{\itemsep}{0pt}

SYSTEM: $\{$
You are an AI assistant. Your task is to answer the TASK as if you are the individual described in the `Persona Profile' (which contains their past survey responses). Remain consistent with the persona's past survey responses and stated characteristics. Carefully follow any instructions provided for the new question, including formatting requirements.
\}

PERSONA DATA: $\{$
\\
Which part of the United States do you currently live in?
\\
Question Type: Single Choice
\\
Options:
\\
1 - Northeast (PA, NY, NJ, RI, CT, MA, VT, NH, ME)
\\
2 - Midwest (ND, SD, NE, KS, MN, IA, MO, WI, IL, MI, IN, OH)
\\
3 - South (TX, OK, AR, LA, KY, TN, MS, AL, WV, DC, MD, DE, VA, NC, SC, GA, FL)
\\
4 - West (WA, OR, ID, MT, WY, CA, NV, UT, CO, AZ, NM)
\\
5 - Pacific (HI, AK)
\\
Answer: 2 - Midwest (ND, SD, NE, KS, MN, IA, MO, WI, IL, MI, IN, OH)
\\
What is the highest level of schooling or degree that you have completed?
\\
Question Type: Single Choice
\\
Options:
\\
1 - Less than high school
\\
2 - High school graduate
\\
3 - Some college, no degree
\\
4 - Associate's degree
\\
5 - College graduate/some postgrad
\\
6 - Postgraduate
\\
Answer: 3 - Some college, no degree
\\
$\ldots$ (Many other survey questions and answers) $\ldots$
\\
Suppose you were given \$5 and had to offer to another (anonymous) person a way to split the money. The other person can either accept or reject your offer. If the
other person accepts your offer, you would each receive the amount you proposed. If the other person rejects your offer, you would both receive 
\$0. How much would you
 offer to the other person?
\\
Question Type: Single Choice
\\
Options:
\\
1 - \$0
\\
2 - \$1
\\
3 - \$2
\\
4 - \$3
\\
5 - \$4
\\
6 - \$5
\\
Answer: 3 - \$2
\\
$\ldots$ (Many other survey questions and answers) $\ldots$
\\
$\}$
\medskip
\\
AD IMAGE: [image]
\medskip
\\
TASK:

Return only one item from ["1","2","3","4","5"] for ad effectiveness.\\
Effective Score Scale Definition:\\
1: Extremely Unlikely. The persona would actively ignore or be annoyed by this ad.\\
2: Unlikely. The persona would likely scroll past without a second thought.
\\
3: Mediocre. It is hard to decide whether the personal would click or don't click.\\
4: Likely. The persona is intrigued and has a good chance of clicking to learn more.\\
5: Extremely Likely. The persona is the ideal target; a click is almost certain.\\
No explanation. Just the score.
\end{tcolorbox}

\caption{Meta-prompt for simulating the effectiveness of a given ad-persona combination.}
\label{fig:persona_prompt}
\end{figure}

Here, in this paper's language, an artifact \(w\) corresponds to [scenario + prompt + its generated ad image]. We denote an evaluation operation of an artifact $w$, which returns a single scalar
\(o \in \mathbb{R}\), as \texttt{Eval(w)}. Repeating \texttt{Eval(w)} with fresh persona samples yields i.i.d. draws
\(o \sim Q_{\mathrm{pers}}(\cdot \mid w)\). As \textsc{TextBO} iteratively improves the prompt for ad image generation, we would like to answer the following question:
\begin{center}
    \textit{How many personas \texttt{Eval(w)} need to be confident that the improved prompt is better than the original prompt?}
\end{center}
In this problem, \(d_W(w,w')\) naturally corresponds to the number of clause-level instruction edits needed to transform prompt \(\pi\) into \(\pi'\). Therefore, we set the resolution \(r=1\), i.e., we require the benchmark to reliably distinguish prompts that differ by at least one step of prompt improvement.

There are many ways of choosing the sampling distribution of $(w, w')$ pairs. \citet{kang2025bayesian} reported results for ten steps of prompt improvement from the initial prompt; we tested ten improvement variants for each step. This totals 100 \textsc{TextBO} prompt improvement cases per scenario. 
We then estimate $\widehat{\kappa}_Q(q)$ per scenario. By Lemma \ref{lem:disc_samplecomplexity_temp0}, the predicted number of independent persona samples per prompt required to decide between two \(r\)-separated artifacts is
\[
L_{\mathrm{req}} \;=\; \left\lceil \frac{2}{\widehat{\kappa}_Q(q)}\log\frac{1}{\delta}\right\rceil.
\]
Table \ref{tab:snr_n} provides \(\widehat{\kappa}_Q(q)\) and $L_{\text{req}}$ for each scenario. Across the scenarios, $L_\text{req}$ varies from 46 to 1180, with an average of 460.25. This shows that the persona sample size choice of 200 is not too bad, but increasing it to 500 could have been a more conservative choice.

\begin{table}[ht!]
\centering
\small
\caption{$\widehat{\kappa}_Q(q)$ and Sample Size for $q=0.05$ and $\delta=0.05$}
\label{tab:snr_n}
\begin{tabular}{lccccccccc}
\toprule
 & Scenario 1 & Scenario 2 & Scenario 3 & Scenario 4 & Scenario 5 & Scenario 6 & Scenario 7 & Scenario 8 & Average\\
\midrule
$\widehat{\kappa}_Q(q)$ 
& 0.00508 
& 0.02320 
& 0.13178 
& 0.00942 
& 0.02955 
& 0.13069 
& 0.0046 
& 0.07548
& 0.051225\\

$L_{\text{req}}$ 
& 1180
& 259
& 46
& 637
& 203
& 46
& 1303
& 80
& 469.25\\
\bottomrule
\end{tabular}
\end{table}

\clearpage

\section{Further discussions}

\subsection{Further discussions on contribution}

\paragraph{Is AO+MB equivalence close to tautological?} The paper’s main contribution is not that the factorization exists once you assume it, but that the “drop-in substitute” claim can be identified exactly, and only, from two protocol-level benchmark-hygiene conditions that are concrete and auditable: what information is exposed to methods (AO) and whether the evaluator is provenance/identity independent (MB). The equivalence is valuable precisely because it collapses an otherwise ambiguous debate (“do personas ‘match’ humans?”) into two checkable failure modes that are independent in practice and routinely violated in real evaluation pipelines. In particular, MB is not a cosmetic assumption: it formalizes that the benchmark is even a well-defined environment, and the paper explicitly discusses why MB is fragile in human/LLM-judge settings unless enforced by design.

Moreover, the result is not “MB alone implies JPC.” The necessity of AO is nontrivial: even if two pipelines induce the same aggregate score distribution, leaking micro-level information can allow adaptive methods to distinguish and exploit differences. The paper includes an explicit counterexample demonstrating that violating AO can break “just panel change” even when the aggregate kernels match (Proposition \ref{prop:AO_counterexample} in Appendix \ref{ssec:AO_counterexample}). That is exactly the kind of conceptual point that is easy to miss without a formal interface-level definition and a necessity proof.

\subsection{Further discussions on extensions}

\paragraph{Would more empirical validation strengthen the arguments in this paper?} The central claim in this paper is an identification theorem about benchmark interfaces, not an empirical claim that a particular persona pipeline tracks humans. The paper includes a proof-of-concept $\mathrm{\kappa Q}$ calibration (Appendix \ref{app:exp-textbo}) to demonstrate the intended workflow: $\mathrm{\kappa Q}$ is designed to be measured and to turn "persona quality" into a budget/design question. Requiring extensive multi-domain experimentation is therefore more about strengthening the paper's applied guidance than about validating the correctness of the main theorem.

\paragraph{Can sample complexity easily be extended beyond pairwise?}
Pairwise comparison is the primitive building block, and this is easily extended beyond pairwise comparison. For example,  If if the workflow is "compare $K$ fixed artifacts," pairwise error control can be achieved with a union bound by allocating per-comparison failure probability $\delta^{\prime}=\delta /C(K,2)$) (or $\delta / \mathrm{K}$ for tournament-style brackets), which yields only a log-factor change in required samples.

\paragraph{}

\end{document}